%% file: main.tex
\documentclass{article} 
\include{header}
\usepackage{iclr2016_conference,times}
\usepackage{amsmath,amssymb,amsthm}
\usepackage{graphicx}
\usepackage{subfigure}
\graphicspath{{./}}
\usepackage{hyperref}
\usepackage{url}
\usepackage{pbox}

\newcommand{\T}{^\top}
\newcommand{\RR}{\mathbb{R}}
\newcommand{\EE}{\mathbb{E}}

\newtheorem{theorem}{Theorem}

\newcommand{\net}{{\rm net}}
\newcommand{\removed}[1]{}
\newcommand{\natinote}[1]{{\Large [[ {\color{red} {#1} -- Nati} ]]}}

\title{Data-Dependent Path Normalization \\in Neural Networks}

\author{Behnam Neyshabur \\
Toyota Technological Institute at Chicago\\
Chicago, IL 60637, USA\\
\texttt{bneyshabur@ttic.edu}                           
\And
Ryota Tomioka\\
Microsoft Research\\
Cambridge, UK\\
\texttt{ryoto@microsoft.com}\hspace{1.2in}
\And
Ruslan Salakhutdinov\\
Department of Computer Science\\
University of Toronto, Canada\\
\texttt{rsalakhu@cs.toronto.edu}\\
\And
Nathan Srebro\\
Toyota Technological Institute at Chicago\\
Chicago, IL 60637, USA\\
\texttt{nati@ttic.edu} \\
}

\newcommand{\var}{\text{Var}}
\newcommand{\cov}{\text{Cov}}

\newcommand{\In}{\text{in}}
\newcommand{\Out}{\text{out}}
\newcommand{\vecX}{\mathbf{x}}
\newcommand{\vecH}{\mathbf{h}}
\newcommand{\vecW}{\mathbf{w}}
\newcommand{\vecY}{\mathbf{y}}
\renewcommand{\vec}[1]{\mathbf{#1}}

\newcommand{\vw}{{\vec w}}
\newcommand{\vx}{{\vec x}}
\newcommand{\vz}{{\vec z}}
\newcommand{\vphi}{\boldsymbol{\phi}}
\newcommand{\vpi}{\boldsymbol{\pi}}
\newcommand{\vecGam}{\boldsymbol{\gamma}}
\newcommand{\vecDelta}{\boldsymbol{\Delta}}
\newcommand{\vecb}{\vec b}
\newcommand{\vecc}{\vec c}
\begin{document}
\maketitle

\begin{abstract}
  We propose a unified framework for neural net normalization,
  regularization and optimization, which includes Path-SGD and
  Batch-Normalization and interpolates between them across two
  different dimensions. Through this framework we investigate the 
issue of
  invariance of the optimization, data dependence and the connection
  with natural gradients.
\end{abstract}

\section{Introduction}
The choice of optimization method for non-convex, over-parametrized
models such as feed-forward neural networks is crucial to the
success of learning---not only does it affect the runtime until
convergence, but it also effects which minimum (or potentially local
minimum) we will converge to, and thus the generalization ability of 
the resulting
model.  Optimization methods are inherently tied to a choice of
geometry over parameter space, which in turns induces a geometry over
model space, which plays an important role in regularization and
generalization \citep{neyshabur15b}.

In this paper, we focus on two efficient alternative optimization
approaches proposed recently for feed-forward neural networks that are
based on intuitions about parametrization, normalization and the
geometry of parameter space: {\bf Path-SGD} \citep{NeySalSre15} was
derived as steepest descent algorithm with respect to particular
regularizer (the $\ell_2$-path regularizer, i.e.~the sum over all
paths in the network of the squared product over all weights in the
path ~\citep{NeyTomSre15}) and is invariant to weight
reparametrization.  {\bf Batch-normalization} \citep{IofSze15} was
derived by adding normalization layers in the network as a way of
controlling the variance of the input each unit receives in a
data-dependent fashion.  In this paper, we propose a unified framework
which includes both approaches, and allows us to obtain additional
methods which interpolate between them.  Using our unified framework,
we can also tease apart and combine two different aspects of these two
approaches: data-dependence and invariance to weight reparametrization.

Our unified framework is based on first choosing a per-node complexity
measure we refer to as $\gamma_v$ (defined in Section
\ref{sec:unified}).  The choice of complexity measure is parametrized
by a choice of ``normalization matrix'' $R$, and different choices for
this matrix incorporate different amounts of data dependencies: for
path-SGD, $R$ is a non-data-dependent diagonal matrix, while for batch
normalization it is a data-dependent covariance matrix, and we can
interpolate between the two extremes.  Once $\gamma_v$ is defined, and
for any choice of $R$, we identify two different optimization
approaches: one relying on a normalized re-parameterization at each
layer, as in batch normalization (Section \ref{sec:bn}), and the other
an approximate steepest descent as in path-SGD, which we refer to as
DDP-SGD (Data Dependent Path SGD) and can be implemented efficiently
via forward and backward propagation on the network (Section
\ref{sec:ddp-sgd}).  We can now mix and match between the choice of
$R$ (i.e.~the extent of data dependency) and the choice of
optimization approach.

One particular advantage of the approximate steepest descent approach
(DDP-SGD) over the normalization approach is that it is invariant to
weight rebalancing (discussed in Section \ref{sec:node-rescaling}).  
This is true
regardless of the amount of data-dependence used.  That is, it
operates more directly on the model (the function defined by the
weights) rather than the parametrization (the values of the weights
themselves).  This brings us to a more general discussion of
parametrization invariance in feedforward networks (Section 
\ref{sec:rescaling}).

Our unified framework and study of in invariances also allows us to
relate the different optimization approaches to Natural Gradients
\citep{Ama98}. In particular, we show that DDP-SGD with full
data-dependence can be seen as an efficient approximation of the
natural gradient using only the diagonal of the Fisher information
matrix (Section \ref{sec:natural}).

\subsection*{Related Works}
There has been an ongoing effort for better understanding of the 
optimization in 
deep networks and several heuristics have been suggested to improve 
the training
 \citep{lecun-98x, 
larochelle2009exploring,difficulty,sutskever2013importance}.
Natural gradient algorithm \citep{Ama98} is known to have a very strong
invariance property; it is not only invariant to reparametrization, but
also to the choice of network architecture. However it is known to
be computationally demanding and thus many approximations have been
proposed \citep{grosse2015scaling,martens2015optimizing,desjardins2015natural}.
However, such approximations make the algorithms less invariant than the
original natural gradient algorithm. \cite{pascanu2013revisiting} also discuss the connections between 
Natural
Gradients and some of the other proposed methods for training neural
networks, namely Hessian-Free Optimization \citep{martens2010deep}, 
Krylov
Subspace Descent \citep{vinyals2011krylov} and TONGA
\citep{roux2008topmoumoute}.

\cite{ollivier2015riemannian} also recently studied the issue of invariance and proposed computationally efficient approximations and alternatives to natural gradient. They study invariances as different mappings from parameter space to the same function space  while we look at the invariances as transformations (inside a fixed parameter space) to which the function is invariant in the model space (see Section~\ref{sec:rescaling}). Unit-wise algorithms suggested in Olivier's work are based on block-diagonal approximations of Natural Gradient in which blocks correspond to non-input units. The computational cost of the these unit-wise algorithms is quadratic in the number of incoming weights. To alleviate this cost, \cite{ollivier2015riemannian} also proposed quasi-diagonal approximations which avoid the quadratic dependence but they are only invariant to affine transformations of activation functions. The quasi-diagonal approximations are more similar to DDP-SGD in terms of computational complexity and invariances (see Section~\ref{sec:node-rescaling}). In particular, ignoring the non-diagonal terms related to the biases in quasi-diagonal natural gradient suggested in \cite{ollivier2015riemannian}, it is then equivalent to diagonal Natural Gradient which is itself equivalent to special case of DDP-SGD when $R_v$ is the second moment (see Table~\ref{tab:framework} and the discussion on relation to the Natural Gradient in Section~\ref{sec:ddp-sgd}).  

\section{Feedforward Neural Nets}\label{sec:ff}
We briefly review our formalization and notation of feedforward neural
nets. We view feedforward neural networks as a parametric class of
functions mapping input vectors to output vectors, where parameters
correspond to {\em weights} on connections between different {\em
  units}.  We focus specifically on networks of ReLUs (Rectified
Linear Units).  Rather than explicitly discussing units arranged in
layers, it will be easier for us (and more general) to refer to the
connection graph as a directed acyclic graph $G(V,E)$ over the set of
node $V$, corresponding to units $v\in V$ in the network.  $V$
includes the inputs nodes $V_{\rm in}$ (which do not have any incoming
edges), the output nodes $V_{\rm out}$ (which do not have any outgoing
edges) and additional internal nodes (possibly arranged in multiple
layers).  Each directed edge $(u\rightarrow v)\in E$ (i.e.~each
connection between units) is associated with a weight $w_{u
  \rightarrow v}$.  Given weight settings $\vecW$ for each edge, the
network implements a function $f_{\vecW}:\RR^{\abs{V_{\rm
      in}}}\rightarrow \RR^{\abs{V_{\rm out}}}$ as follows, for any
input $\vecX\in\RR^{\abs{V_{\rm in}}}$:
\begin{itemize}
\item For the input nodes $v\in V_{\rm in}$, their output $h_v$ is the
  corresponding coordinate of the input $\vecX$.
\item For each internal node $v$ we define recursively
  $z_v=\sum_{(u\rightarrow v)\in E} w_{u\rightarrow v}\cdot h_u$ and
  $h_v=[z_v]_+$ where $[z]_+=\max(z,0)$ is the ReLU activation function
  and the summation is over all edges incoming into $v$.
\item For output nodes $v\in V_{\rm out}$ we also have
  $z_v=\sum_{(u\rightarrow v)\in E} w_{u\rightarrow v}\cdot h_u$, and
  the corresponding coordinate of the output $f_\vecW(\vecX)$ is given
  by $z_v$.  No non-linearity is applied at the output nodes, and the
  interpretation of how the real-valued output corresponds to the
  desired label is left to the loss function (see below).
\item In order to also allow for a ``bias'' at each unit, we can
  include an additional special node $v_{\rm bias}$ that is connected
  to all internal and output nodes, where $h_{v_{\rm bias}}=1$ always
  ($v_{\rm bias}$ can thus be viewed as an additional input node whose
  value is always $1$).
\end{itemize}

We denote $N^{\In}(v)=\left\{ u \middle| (u\rightarrow v)\in
  E\right\}$ and $N^{\Out}(v)=\left\{ u \middle| (v\rightarrow u)\in
  E\right\}$, the sets of nodes feeding into $v$ and that $v$ feeds
into.  We can then write $\vecH_{N^\In(v)}\in\RR^{\left|
    N^\In(v)\right|}$ for the vector of outputs feeding into $v$, and
$\vecW_{\rightarrow v}\in\RR^{\left| N^\In(v)\right|}$ for the vector
of weights of unit $v$, so that $z_v=\inner{\vecW_{\rightarrow
    v}}{\vecH_{N^\In(v)}}$.  
        
We do not restrict to layered networks, nor do we ever need to
explicitly discuss layers, and can instead focus on a single node at a
time (we view this as the main advantage of the graph notation).  But
to help those more comfortable with layered networks understand the
notation, let us consider a layered fully-connected network:  The
nodes are partitioned into layers $V=V_0 \cup V_1 \cup \ldots
V_d$, with $V_{\rm in}=V_0$, $V_{\rm out}=V_d$. For all nodes $v\in
V_i$ on layer $i$, $N^\In(v)$ is the same and equal to
$N^\In(v)=V_{i-1}$, and so $\vecH_{N^\In(v)}=\vecH_{V_{i-1}}$ consists
of all outputs from the previous layer and we recover the layered
recursive formula $h_v=[\inner{\vecW_{\rightarrow
    v}}{\vecH_{V_{i-1}}}]_+$ and $\vecH_{V_{i}}=[\mathbf{W}_i
\vecH_{V_{i-1}}]_+$, where $\mathbf{W}\in
\RR^{\left|V_i\right|\times\left|V_{i-1}\right|}$ is a matrix with
entries $w_{u\rightarrow v}$, for each $u\in V_{i-1}, v\in V_i$.  This
description ignores the bias term, which could be modeled as a
direct connection from $v_{\rm bias}$ into every node on every layer,
or by introducing a bias unit (with output fixed to 1) at each layer. Please
see Figure~\ref{fig:notation} for an example of a layered feedforward
network and a summary of notation used in the paper.

\begin{figure}[tb]
\begin{center}
\hspace{-0.3in}
\subfigure{
\includegraphics[width=1\textwidth]{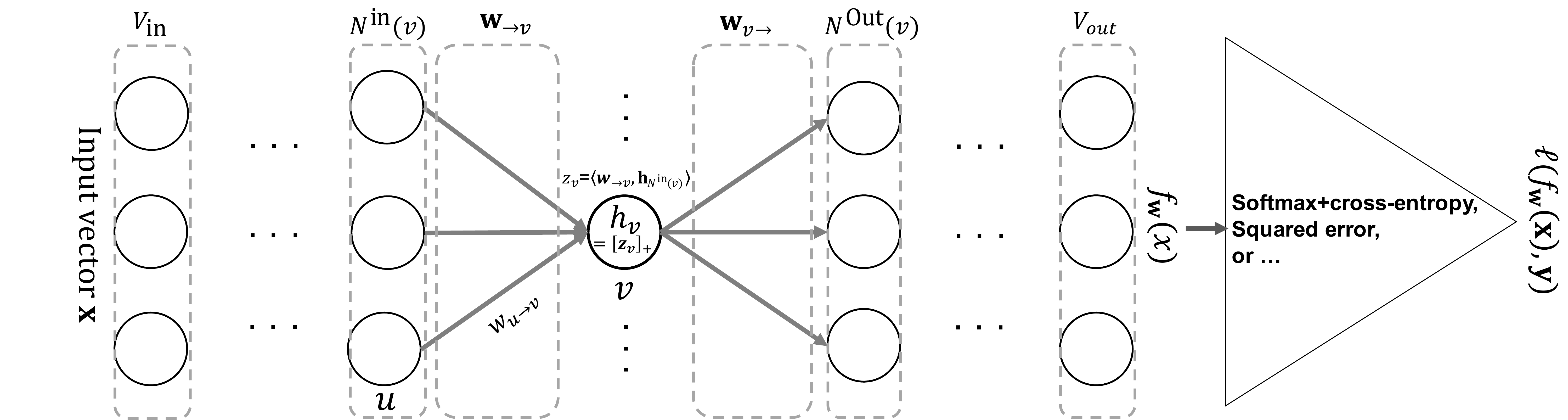}
}
\subfigure{
\small
\setlength\tabcolsep{4pt}
\begin{tabular}[tb]{|c|c|c|c|}
\hline
Symbol & Meaning &  Symbol
& Meaning\\
\hline
$\vecX$ / $\vec y$& input vector / label&  $V_\In$ / $V_\Out$& the set of input / output nodes\\
\hline
$\vecW$ & the parameter vector& $w_{u\rightarrow v}$
& the weight of the edge $(u\rightarrow v)$\\
\hline
$\vecW_{v\rightarrow}$ & the vector of incoming  weights to $v$ & $\vecW_{\rightarrow v}$
& the vector of outgoing  weights from $v$\\
\hline
$N^{\In}(v)$ & the set of nodes feeding into $v$ & $N^{\Out}(v)$
& the set of nodes that $v$ feeds into\\
\hline
$h_v$ & the output value of node $v$ & $z_v$
& the activation value of node $v$\\
\hline
\end{tabular}
}
  \caption{An example of layered feedforward networks and notation used in the paper}
  \label{fig:notation}
 \end{center}
\end{figure}

\setcounter{table}{0}

We consider supervised training tasks, where each input $\vecX$ is
associated with a desired label $y$ and how well the network captures
this label is quantified by a loss function
$\ell(f_{\vecW}(\vecX),y)$.  For example, in a classification problem
$y$ is one of $|V_{\rm out}|$ classes and a cross-entropy (soft-max)
loss might be used.  We also refer to a source distribution
$\D_{(\mathbf{x},\mathbf{y})}$ over input-label pairs, where the goal
is to minimize the expected loss:
\begin{equation}
L_{\D}(\vecW)=\E_{(\mathbf{x},y) \sim \D}
\left[ \ell(f_{\vecW}(\mathbf{x}),y) \right]
\end{equation}
All expectations, unless specified otherwise, refer to expectation
w.r.t.~this source distribution.

\subsection*{Invariances and Node-wise Rescaling}

Once the architecture (graph $G$) is fixed, every choice of weight
$\vecW$ defines a function $f_\vecW$.  But this parameterization is
not unique--the same function $f$ could be parameterized by two
different weight settings (i.e.~we could have $f_\vecW=f_{\vecW'}$
even though $\vecW\not=\vecW'$).  Ideally, we'd like to work as
directly as possibly on the functions rather then the
parameterization.  It is therefor important to understand different
``invariances'', i.e.~different transformations that can be applied to
the weights without changing the function.  We note that the notion of
invariance we define here is tied to a fixed network architecture $G$:
we do not consider transformation that changes the network
architecture, such as insertion of a linear layer
\cite[as in, e.g.][]{ollivier2015riemannian}.

We say that network $G$ is invariant to an invariant transformation
$T(\vecW)$, iff for any weight setting $\vecW$,
$f_{\vecW}=f_{T(\vecW)}$.  We say that an update rule $\mathcal{A}$
(e.g.~a rule for obtaining the next iterate from the current iterate
in an optimization procedure) is invariant to transformation $T$ iff
for any weight setting $\vecW$,
$f_{\mathcal{A}(\vecW)}=f_{\mathcal{A}(T(\vecW))}$.  That is, whether
we start an iterative optimization procedure using updates
$\mathcal{A}$ at $\vecW$ or the at the equivalent $\vecW'=T(\vecW)$,
we would always be working on the same function (only with a different
parameterization).

An important invariance in feedforward ReLU network is {\em node-wise
  rescaling} (or rebalancing). For any positive scalar $\rho$ and for
any internal node $v$ ($v\notin V_{\rm in}$ and $v\notin V_{\rm
  out}$), we can scale all the incoming weights into $v$ by $\rho$ and
all the outgoing weights by $1/\rho$ without changing the computation
in the networks.  That is, the following transformation
$\vecW'=T(\vecW)$ satisfies $f_{\vecW}=f_{\vecW'}$:
\begin{equation}
  \label{eq:node-wise}
\begin{aligned}
 w'_{v\rightarrow u} &= \rho w_{v\rightarrow u}\quad  (\forall u\in
N^{\Out}(u)),\\
 w'_{u\rightarrow v} &=\rho^{-1} w_{u\rightarrow v}\quad
 (\forall u\in N^\In(v))\\
 w'_{u\rightarrow u'} &=w_{u\rightarrow u'} \quad (\textrm{otherwise})
\end{aligned}
\end{equation}
We can combined multiple such rescalings to push the scaling up or
down the network without changing the computation.  One of our goals
is obtaining optimization algorithms that are invariant to such
transformations.


\begin{table}[tb]
\small
 \begin{center}
 \small
\setlength\tabcolsep{3pt}
\begin{tabular}[tb]{|c|c|c|c|}
\hline
$R_v$ & \pbox{20cm}{Measure} &  Normalized reparametrization
& Diagonal steepest descent\\
\hline
$D=\diag\left(\vecGam^2_{N^{\In(v)}}\right)$ & Path-Norm & Unit-wise Path-Normalization
& {\bf Path-SGD}\\
\hline
$C=\cov\left(\vecH_{N^{\In(v)})}\right)$ & Variance
& {\bf Batch-Normalization} & \\
\hline
$M=\E\left[\vecH_{N^{\In(v)})} \vecH_{N^{\In(v)})}^\top\right]$ & Second Moment & 
& {\bf Diag. Natural Gradient}\\
\hline
\pbox{20cm}{$\alpha M+ (1-\alpha)D$ \\ $\alpha C+ (1-\alpha)D$}
& DDP-Norm & DDP-Normalization
& DDP-SGD\\
\hline
Node-wise Rescaling Invariant & Yes
& No
& Yes\\
\hline
\end{tabular}
 \end{center}
 \caption{Some of the choices for $R_v$ in the proposed unified framework.}
 \label{tab:framework}
\end{table}

\section{A Unified Framework}\label{sec:unified}
We define a complexity measure on each node as follows:
\begin{equation}\label{eq:R}
\gamma_v(\vecW) = \sqrt{\vecW_{\rightarrow v}^\top R_v\vecW_{\rightarrow v} }
\end{equation}
where $R_v$ is a positive semidefinite matrix that could depend on the
computations feeding into $v$, and captures both the complexity of the
nodes feeding into $v$ and possibly their interactions.  We consider
several possibilities for $R_v$, summarized also in Table 1.  

A first possibility is to set
$R_v=\diag\left(\vecGam^2_{N^{\In(v)}}\right)$ to a diagonal matrix
consisting of the complexities of the incoming units.  This choice
does not depend on the source distribution (i.e.~the data), and also
ignores the effect of activations (since the activation pattern
depends on the input distribution) and of dependencies between
different paths in the network.  Intuitively, with this choice of
$R_v$, the measure $\gamma_v(\vecW)$ captures the ``potential'' (data
independent) variability or instability at the node.

Another possibility is to set $R_v$ to either the covariance
(centralized second moment) or to the (uncentralized) second moment
matrix of the outputs feeding into $v$.  In this case,
$\gamma^2_v(\vecW)$ would evaluate to the variance or (uncentralized)
second moment of $z_v$.  We could also linearly combined the data
independent measure, which measures inherent instability, with one of
these the data-dependent measure to obtain:
\begin{equation}\label{eq:gamma}
\gamma^2_v(\vecW) = \alpha S(z_v) + (1-\alpha)\sum_{u \in 
N^{\In} (v)}\gamma^2_u(\vecW) w^2_{u\rightarrow v}\quad (v\notin V_{\rm 
in}),
\end{equation}
where $S(z_v)$ is either the variance or uncentralized second moment,
and $\alpha$ is a parameter.

The complexity measure above is defined for each node of the network
separately, and propagates through the network.  To get an overall
measure of complexity we sum over the output units and define the
following complexity measure for the function $f_{\vecW}$ as
represented by the network:
\begin{equation}
\gamma^2_{\rm net}(\vecW) = \sum_{v\in V_{\rm out}}\gamma^2_v(\vecW).
\end{equation}

For $R_v=\diag\left(\vecGam^2_{N^{\In(v)}}\right)$, this complexity
measure agrees with the $\ell_2$-Path-regularizer as introduced by
\cite{NeyTomSre15}.  This is the sum over all paths in the network of the squared
product of weights along the path.  The path-regularizer is also
equivalent to looking at the minimum over all ``node rescalings'' of
$\vecW$ (i.e.~all possibly rebalancing of weights yielding the same
function $f_\vecW$) of the $\max_v \norm{\vecW_{\rightarrow v}}$.
But, unlike this max-norm measure, the path-regularizer does {\em not} depend
on the rebalancing and is invariant to node rescalings \citep{NeyTomSre15}.

For data-dependent choices of $R_v$, we also get a similar invariance
property.  We refer to the resulting complexity measure,
$\gamma^2_{\rm net}(\vecW)$, as the Data-Dependent-Path (DDP) regularizer.

\removed{
\natinote{Do we want to explain that it is a norm when we consider it
  over functions?  If so, we should include back the definition of
  $\gamma(f)$ which I removed.   Otherwise, I don't think we need it}.
}

\removed{
Such a measure induces a complexity measure on functions $f$ that can 
be presented by network $G$.
The complexity of $f$ based on DDP-regularizer over $G$
can then be defined as the complexity of the simplest network that 
represent the function:
\begin{equation}
\gamma(f) =\inf_{f_{G,\vecW}=f} \gamma_{\rm net}(\vecW)
\end{equation}
}

After choosing $R_v$, we will think of $\gamma_v$ as specifying the
basic ``geometry'' and bias (for both optimization and learning) over
weights.  In terms of learning, we will (implicitly) prefer weights
with smaller $\gamma_v$ measure, and correspondingly in terms of
optimization we will bias toward smaller $\gamma_v$ ``balls''
(i.e.~search over the part of the space where $\gamma_v$ is smaller).
We will consider two basic ways of doing this: In Section \ref{sec:bn}
we will consider methods that explicitly try to keep $\gamma_v$ small
for all internal nodes in the network, that is explicitly search over
simpler weights.  Any scaling is pushed to the output units, and this
scaling hopefully does not grow too much due.  In Section \ref{sec:ddp-sgd}
we will consider (approximate) steepest descent methods with respect
to the overall $\gamma_{\rm net}$, i.e.~updates that aim at improving
the training objective while being small in terms of their effect on
$\gamma_{\rm net}$.

\removed{

We can see from \eqref{eq:matform} that this is a specific way of
defining matrix $R_{N^\In(v)}$ in definition \eqref{eq:R}.
%
regularized second moment which combines a complexity measure with 
second moment of each node and this regularized second moment is 
propagating and will be used to calculate the regularized second 
moment for higher layers.
While we will be focusing on this measure, it is also possible to 
define $R_{N^\In(v)}$ based on the geometric mean instead of the 
arithmetic mean:
\begin{equation}
R_{N^\In(v)} = \E\left[ \vecH_{N^\In(u)} \vecH_{N^\In(u)}^\top\right]^\alpha 
\diag\left(\vecGam^2_{N^\In(u)}\right)^{(1-\alpha)}
\end{equation}

}

\section{DDP-Normalization: A Batch-Normalization Approach}\label{sec:bn}

In this Section, we discuss an optimization approach based on ensuring
$\gamma_v$ for all internal nodes $v$ are fixed and equal to
one---that is, the complexity of all internal nodes is ``normalized'',
and any scaling happens only at the output nodes.  We show that with a
choice of $R_v=\cov\left(\vecH_{N^{\In(v)})}\right)$, this is
essentially equivalent to Batch Normalization \citep{IofSze15}.

Batch-Normalization \cite{IofSze15} was suggested as an alternate
architecture, with special ``normalization'' layers, that ensure the
variance of node outputs are normalized throughout training.
Considering a feed-forward network as a graph, for each node $v$, the
Batch-Normalization architecture has as parameters an (un-normalized)
incoming weight vector $\tilde{\vecW}$ and two additional scalars
$c_v,b_v\in\R$ specifying scaling and shift respectively.  The
function computed by the network is then given by a forward
propagation similar to standard feed-forward ReLU networks as
described in Section \ref{sec:ff}, except that for each node an un-normalized
activation is first computed:
\begin{equation}
  \label{eq:ztilde}
  \tilde{z}_v = \inner{\tilde{\vecW}_{\rightarrow v}}{\vecH_{N^{\In}(v)}}
\end{equation}
Then, this activation is normalized to obtain the normalized
activation, which is also scaled and shifted, and the output of the
unit is the output of the activation function for this activation value:
\begin{equation}\label{eq:bn}
  \begin{gathered}
z_v=c_v
\frac{\tilde{z}_v-\E[\tilde{z_v}]}{\sqrt{\var(\tilde{z}_v)}}+b_v \\
    h_v=[z_v]_+
  \end{gathered}
\end{equation}
The variance and expectation are actually calculated on a
``mini-batch'' of training examples, giving the method its name.
Batch-normalization then proceeds by training the architecture
specified in \eqref{eq:ztilde} and \eqref{eq:bn} through mini-batch stochastic
gradient descent, with each gradient mini-batch also used for
estimating the variance and expectation in \eqref{eq:bn} for all
points in the mini-batch.

Instead of viewing batch-normalization as modifying the architecture,
or forward propagation, we can view it as a re-parameterization, or
change of variables, of the weights in standard feed-forward networks
as specified in Section \ref{sec:ff}.  In particular, instead of specifying
the weights directly through $\vecW$, we specify them through
$\tilde{\vecW},\vecb$ and $\vecc$, with the mapping:
\begin{align}
  &\tilde{\gamma}^2_v = \tilde{\vecW}_{\rightarrow v}^\top
  R_v\tilde{\vecW}_{\rightarrow v} \quad\quad R_v=\cov(h_{N^{\In}(v)})
  \label{eq:tildegamma}\\
&w_{u \rightarrow v} = 
\begin{cases}
c \frac{\tilde{w}_{u \rightarrow v}}{\tilde{\gamma}_v} & u\neq 
v_{\text{bias}}\\
b-c \frac{\E\left[\inner{\tilde{\vecW}_{\rightarrow v}}{\vecH_{N^{\In}(v)}}\right]}{\tilde{\gamma}_v}& u=v_{\text{bias}}\\
\end{cases}\label{eq:reparam}
\end{align}
The model class of functions used by Batch-Normalization is thus
exactly the same model class corresponding to standard feed-forward
network, just the parameterization is different.  However, the change
of variables from $\vecW$ to $\tilde{\vecW},\vecb,\vecc$ changes the
geometry implied by the parameter space, and consequently the
trajectory (in model space) of gradient updates---effectively
transforming the gradient direction by the Jacobian between the two
parameterizations.  Batch-Normalization can thus be viewed as an
alternate optimization on the same model class as standard
feed-forward networks, but with a different geometry.  The
reparametrization ensures that $\gamma_v(\vecW)=c_v$ for all
nodes---that is, the complexity is explicit in the parameterization
and thus gets implicitly regularized through the implicit
regularization inherent in stochastic gradient updates.

The re-parameterization \eqref{eq:reparam} is redundant and includes
more parameters than the original parameterization $\vecW$---in
addition to one parameter per edge, it includes also two additional
parameters per node, namely the shift $b_v$ and scaling $c_v$.  The
scaling parameters at internal nodes can be avoided and removed by
noting that in ReLU networks, due to the node-rescaling property, all
scaling can be done at the output nodes.  That is, fixing $c_v=1$ for
all internal $v$ does not actually change the model class (all
functions realizable by the model can be realized this way).
Similarly, we can also avoid the additional shift parameter $b_v$ and
rely only on bias units and bias weights $\tilde{w}_{v_{\rm
    bias}\rightarrow v}$ that get renormalized together with weights.
The bias term $\tilde{w}_{v_{\rm bias}\rightarrow v}$ does {\em not}
affect normalization (since it is deterministic and so has no effect
on the variance), it just gets rescaled with the other weights.

We thus propose using a simpler reparametrization (change of
variables), with the same number of parameters, using only
$\tilde{\vecW}$ and defining for each internal unit:
\begin{equation}\label{eq:ddp-reparam}
w_{u\rightarrow v}=\frac{\tilde{w}_{u\rightarrow v} }{\tilde{\gamma}_v}
\end{equation}
with $\tilde{\gamma}_v$ as in \eqref{eq:tildegamma}, and with the
output nodes un-normalized: $\vecW_{\rightarrow V_{\rm out}} =
\tilde{\vecW}_{\rightarrow V_{\rm out}}$.  This ensures that for all
internal nodes $\gamma_v(\vecW)=1$.


Going beyond Batch-Normalization, we can also use the same approach
with other choices of $R_v$, including all those in Table 1: We
work with a reparametrization $\tilde{\vecW}$, defined through
\eqref{eq:tildegamma} and \eqref{eq:ddp-reparam} but with different
choices of $R_v$, and take gradient (or stochastic gradient) steps
with respect to $\tilde{\vecW}$.  Expectations in the definition of
$R_v$ can be estimated on the stochastic gradient descent mini-batch
as in Batch-Normalization, or on independent samples of labeled or
unlabeled examples.  We refer to such methods as ``DDP-Normalized''
optimization.  Gradients in DDP-Normalization can be calculated
implemented very efficiently similar to Batch-Normalization (see
Appendix \ref{sec:BN-imp}).

When using this type of DDP-Normalization, we ensure that for any
internal node $\gamma_v(\vecW)=1$ (the value of $\tilde{\gamma}_v$ can
be very different from $1$, but what is fixed is the value of
$\gamma_v$ as defined in \eqref{eq:R} in terms of the weights $\vecW$,
which in turn can be derived from $\tilde{\vecW}$ through
\eqref{eq:reparam}), and so the overall complexity $\gamma_{\rm
  net}(\vecW)$ depends only on the scaling at the output layer.

Another interesting property of DDP-Normalization updates is that for
any internal node $v$, the updates direction of
$\tilde{w}_{\rightarrow v}$ is exactly orthogonal to the weights:
\begin{theorem}\label{thm:orthogonal}
For any weight $\tilde{\vecW}$ in DDP-Normalization and any non-input node $v\notin V_{\In}$
\begin{equation}
\inner{\tilde{\vecW}_{\rightarrow v}}{ \frac{ \partial L }{\partial \tilde{\vecW}_{\rightarrow v} } }=0
\end{equation} 
\end{theorem}
The fact that the gradient is orthogonal to the parameters
means weight updates in DDP-Normalization are done
in a way that it prevents the norm of weights to change considerably 
after each updates (the proof is given in Appendix \ref{sec:proofs}).

\section{DDP-SGD}\label{sec:ddp-sgd}

We now turn to a more direct approach of using our complexity measure
for optimization.  To do so, let us first recall the strong connection
between geometry, regularization and optimization through the specific
example of gradient descent.  

Gradient descent can be thought of as steepest descent with respect to
the Euclidean norm---that is, it takes a step in a direction that
maximizes improvement in the objective while also being small in terms
of the Euclidean norm of the step.  The step can also be viewed as a
regularized optimization of the linear approximation given by the
gradient, where the regularizer is squared Euclidean norm.  Gradient
Descent is then inherently linked to the Euclidean norm---runtime of
optimization is controlled by the Euclidean norm of the optimum and
stochastic gradient descent yields implicit Euclidean norm
regularization.  A change in norm or regularizer, which we think of as
a change of geometry, would then yield different optimization
procedure linked to that norm.

What we would like is to use the DDP-regularizer $\gamma_{\rm \net}(\vecW)$
to define our geometry, and for that we need a distance (or
divergence) measure corresponding to it by which we can measure the
``size'' of each step, and require steps to be small under this
measure.  We cannot quite do this, but instead we use a diagonal
quadratic approximation of $\gamma_{\rm net}(\vecW)$ about our current
iterate, and then take a steepest descent step w.r.t.~the quadratic
norm defined by this approximation.

Specifically, given a choice of $R_v$ and so complexity measure
$\gamma_{\rm net}(\vecW)$, for the current iterate $\vecW^{(t)}$ we define the
following quadratic approximation:
\begin{equation}
\hat{\gamma}^2_{\rm net}(\vecW^{(t)}+\Delta \vecW) = \gamma^2_{\rm net}(\vecW^{(t)})+
\inner{\nabla \gamma^2_{\rm net}(\vecW^{(t)})}{\Delta \vecW} +
\frac{1}{2} \Delta \vecW^\top
\diag\left(\nabla^2  \gamma^2_{\rm net}(\vecW^{(t)})\right) \Delta \vecW
\end{equation}
and the corresponding quadratic norm:
\begin{equation}
\norm{\vecW'-\vecW}^2_{\hat{\gamma}^2_{\rm net} } =
\norm{\vecW'-\vecW}^2_{\diag(\frac{1}{2}\nabla^2 \gamma^2_{\rm
    net}(\vecW^{(t)}))}=\sum_{(u\rightarrow v)\in G} 
\frac{1}{2}\frac{\partial^2 \gamma^2_{\rm net}}{\partial \vecW^2_{u\rightarrow v}} (\vecW'_{u\rightarrow v}-\vecW_{u\rightarrow v})^2.
\end{equation}
We can now define the DDP-update as:
\begin{equation}\label{eq:ddp}
\vecW^{(t+1)}=\min_{\vecW} \eta \inner{\nabla L(w)}{\vecW-\vecW^{(t)}} + 
\frac{1}{2} \norm{\vecW'-\vecW}^2_{\hat{\gamma}^2_{\rm net} }.
\end{equation}
Another way of viewing the above approximation is as taking a diagonal
quadratic approximation of the Bergman divergence of the regularizer.
Solving \eqref{eq:ddp} yields the update:
\begin{equation}
w^{(t+1)}_{u\rightarrow v} = w_{u\rightarrow v} - \frac{\eta}{\kappa_{u\rightarrow v}(\vecW)} \frac{\partial 
L}{\partial w_{u\rightarrow v}}(\vecW^{(t)}) \quad\quad\textrm{where: } \kappa_{u\rightarrow v}(\vecW)=\frac{1}{2}\frac{\partial^2 \gamma^2_{\rm 
net}}{\partial w^2_{u\rightarrow v}}.
\end{equation}
Instead of using the full gradient, we can also use a limited number
of training examples to obtain stochastic estimates of $\frac{\partial 
L}{\partial w_{u\rightarrow v}}(\vecW^{(t)})$---we refer to the
resulting updates as DDP-SGD.

For the choice $R_v=\diag(\gamma^2_{N^{\In}(v)})$, we have that
$\gamma^2_{\rm net}$ is the Path-norm and we recover Path-SGD
\cite{NeySalSre15}. As was shown there, the Path-SGD updates can be
calculated efficiently using a forward and backward propagation on the
network, similar to classical back-prop.  In Appendix \ref{sec:ddp-imp} we show
how this type of computation can be done more generally also for other
choices of $R_v$ in Table 1.

\subsection*{Relation to the Natural Gradient}\label{sec:natural}

The DDP updates are similar in some ways to Natural Gradient updates,
and it is interesting to understand this connection.  Like the DDP,
the Natural Gradients direction is a steepest descent direction, but
it is based on a divergence measure calculated directly on the
function $f_{\vecW}$, and not the parameterization $\vecW$, and as
such is invariant to reparametrizations.  The natural gradient is
defined as a steepest descent direction with respect to the
KL-divergence between probability distributions, and so to refer to it
we must refer to some probabilistic model.  In our case, this will be
a conditional probability model for labels $\vecY$ conditioned on the
inputs $\vecX$, taking expectation with respect to the true marginal
data distribution over $\vecX$.

What we will show that for the choice
$R_v=\E[\vecH_{N^{\In}(v)}\vecH_{N^{\In}(v)}^\top]$, the DDP update
can also be viewed as an approximate Natural Gradient update.  More
specifically, it is a diagonal approximation of the Natural Gradient
for a conditional probability model $q(\vecY| \vecX;\vecW)$ (of the labels
$\vecY$ given an input $\vecX$) parametrized by $\vecW$ and specified
by adding spherical Gaussian noise to the outputs of the network:
$\vecY|\vecX\sim \mathcal{N}(f_\vecW(\vecX),I_{|V_{\rm out}|})$.

Given the conditional probability distribution $q(\vecY|\vx;\vecW)$, we can
calculate the expected Fisher information matrix.  This is a matrix
indexed by parameters of the model, in our case edges $e=(u\rightarrow
v)$ on the graph
and their corresponding weights $w_e$, with entries defined as follows:
 \begin{align}
  \label{eq:fisher-information-m}
 F(\vw)[e,e'] = \EE_{\vx\sim p(\vecX)}\EE_{\vecY\sim q(\vecY|\vx;\vecW)}\left[
\frac{\partial \log q(\vecY|\vx;\vecW)}{\partial w_e}
\frac{\partial \log q(\vecY|\vx;\vecW)}{\partial w_{e'}}
 \right],
 \end{align}
where $x\sim p(\vecX)$ refers to the marginal source distribution (the data
distribution).  That is, we use the true marginal distributing over
$\vecX$, and the model conditional distribution $\vecY|\vecX$, ignoring
the true labels. The Natural
Gradient updates can then be written as(see appendix \ref{sec:ng} for more 
information):
\begin{equation}\label{eq:ng-m}
  \vecW^{(t+1)} = \vecW^{(t)} - \eta F(\vecW^{(t)})^{-1} \nabla_\vecW L(\vecW^{(t)}).
\end{equation}

If we approximate the Fisher information matrix
with its diagonal elements, the update step normalizes each dimension 
of the gradient with the corresponding element on the diagonal of 
the Fisher information matrix:
\begin{equation}\label{eq:dng}
w^{(t+1)}_{e} = w^{(t)}_{e} - \frac{\eta}{F(\vecW)[e,e]} \frac{\partial L}{\partial 
w_{e}}(\vecW^{(t)}).
\end{equation}

Using diagonal approximation of Fisher information matrix to normalize the gradient values has been suggested before as a computationally tractable alternative to the full Natural Gradient \citep{lecun1998neural,schaul2013no}.  \cite{ollivier2015riemannian} also suggested a ``quasi-diagonal" approximations that includes, in addition to the diagonal, also some non-diagonal terms corresponding to the relationship between the bias term and every other incoming weight into a unit.

For our Gaussian probability model, where $\log
 q(\vecY|\vecX)=\frac{1}{2}\norm{\vecY-f_\vecW(\vecX)}^2+{\rm 
const}$, the diagonal can be calculated as:
\begin{equation}\label{eq:diag}
F(\vecW)[e,e] = \E_{\vecX\sim
 p(\vecX)}\left[\sum_{v'\in V_\Out} \left(\frac{\partial 
f_\vecW(\vecX)[v']}{\partial w_e}\right)^2\right],
\end{equation}
using \eqref{eq:partial-logq}. We next prove that this update is equivalent to
DDP-SGD for a specific choice of $R_v$, namely the second moment.
\begin{theorem}
The Diagonal Natural Gradient indicated in equations~\eqref{eq:dng} and ~\eqref{eq:diag} is equivalent to DDP-SGD for $R_v=\E\left[\vecH_{N^{\In}(v)}\vecH_{N^{\In}(v)}^\top\right]$.
\end{theorem}
\begin{proof}
We calculate the scaling factor $\kappa_{u\rightarrow v}(\vecW)$ for 
DDP-SGD as follows:
\begin{align*}
\kappa_{u\rightarrow v}(\vecW)&=\frac{1}{2}\frac{\partial^2 \gamma^2_{\rm net}}{\partial w_{u\rightarrow v}^2}=\frac{1}{2}\sum_{v'\in V_{\Out}}\frac{\partial^2 \E[z^2_{v'}]}{\partial w_{u\rightarrow v}^2} = \sum_{v'\in V_{\Out}} \frac{\partial}{\partial w_{u\rightarrow v}}\left(\frac{1}{2}\frac{\partial \E[z^2_{v'}]}{\partial w_{u\rightarrow v}}\right)\\
&= \sum_{v'\in V_{\Out}}\frac{\partial}{\partial w_{u\rightarrow v}}\left(\E\left[ z_{v'} \frac{\partial z_{v'}}{\partial w_{u\rightarrow v}}\right]\right)
= \sum_{v'\in V_{\Out}}\frac{\partial}{\partial w_{u\rightarrow v}}\left(\E\left[ z_{v'} h_u\frac{\partial z_{v'}}{\partial z_{v}}\right]\right)\\
&= \sum_{v'\in V_{\Out}}\E\left[ h^2_u\left(\frac{\partial z_{v'}}{\partial z_{v}}\right)^2\right] = \E\left[h^2_u \sum_{v'\in V_{\Out}} \left(\frac{\partial z_{v'}}{\partial z_{v}}\right)^2\right]\\
&=\E\left[\sum_{v'\in V_\Out} \left(\frac{\partial 
f_\vecW(\vecX)[v']}{\partial w_e}\right)^2\right] = F(\vecW)[u\rightarrow v,u\rightarrow v]\\
\end{align*}
Therefore, the scaling factors in DDP-SGD  with $R_v=\E\left[\vecH_{N^{\In}(v)}\vecH_{N^{\In}(v)}^\top\right]$
are exactly the diagonal elements of the Fisher Information matrix used 
in the Natural Gradient updates.
\end{proof}

\removed{
In classification tasks, we usually use softmax activations and it that case, the 
diagonal of the Fisher Information can be calculated as:
\begin{equation}
 F(\vecW)[u\rightarrow v,u\rightarrow v] = \E\left[h^2_{u}
  \sum_{v'\in V_{\rm out}} \sum_{v''\in V_{\rm out}}
  \left(1_{v'=v''}\cdot\sigma_{\text{soft}}(z_{v'})  
-\sigma_{\text{soft}}(z_{v'})\sigma_{\text{soft}}(z_{v''})\right)
 \frac{\partial  z_{v'}}{\partial z_v}\cdot
 \frac{\partial z_{v''}}{\partial z_v}
  \right]
\end{equation}
where $\sigma_{\text{soft}}$ is the softmax function and we have that 
$\sigma_{\text{soft}}(z_v) = \frac{e^{z_v}}{\sum_{v'\in V_\Out} 
e^{z_{v'}}}$. The above values can be calculated as fast as 
$\abs{V_\Out}$ backpropagation on the mini-batch. Moreover, 
considering the connection between DDP-SGD and the diagonal of outer 
product of Jacobians, we can consider adding the path-regularizer to 
the diagonal of the Fisher information.
}
\section{Node-wise invariance} \label{sec:node-rescaling} In this
section, we show that DDP-SGD is invariant to node-wise rescalings
(see Section \ref{sec:ff}), while DDP-Normalization does not have favorable
invariance properties.

\subsection{DDP-SGD on feedforward networks}
In Section \ref{sec:ff}, we observed that feedforward ReLU networks are
invariant to node-wise rescaling.  To see if DDP-SGD is also invariant
to such rescaling, consider a rescaled $\vecW'=T(\vecW)$, where $T$ is
a rescaling by $\rho$ at node $v$ as in \eqref{eq:node-wise}.  Let
$\vecW^+$ denote the weights after a step of DDP-SGD.  To establish
invariance to node-rescaling we need to show that
$\vecW'^+=T(\vecW^+)$.  For the outgoing weights from $v$  we have:
\begin{align*}
w'^{+}_{v\rightarrow j} &= \rho w_{v\rightarrow j} - 
\frac{\rho^2 \eta}{\kappa_{v\rightarrow j}(\vecW)}\frac{\partial 
L}{\rho \partial w_{v\rightarrow j}}(\vecW)\\
&=\rho\left(w_{v\rightarrow j}- \frac{\eta}{\kappa_{v\rightarrow 
j}(\vecW)}\frac{\partial L}{\partial w_{v\rightarrow j}}(\vecW)\right)
 = \rho w^{+}_{v\rightarrow j}\\
\end{align*}
Similar calculations can be done for incoming weights to the node $v$.
The only difference is that $\rho$ will be substituted by $1/\rho$. Moreover,
note that due to non-negative homogeneity of ReLU activation function,
the updates for the rest of the weights remain exactly the same.
Therefore, DDP-SGD is node-wise rescaling invariant.

\subsection{SGD on DDP-Normalized networks}
Since DDP-Normalized networks are reparametrization of feedforward
networks, their invariances are different. Since the operations in DDP-Normalized
networks are based on $\tilde{w}$, we should study the invariances for
$\tilde{w}$.  The invariances in this case are given by rescaling of
incoming weights into a node, i.e. for an internal node $v$ and
scaling $\rho>0$:
\begin{equation}
T(\tilde{w})_{k\rightarrow v} =\rho \tilde{w}_{k\rightarrow v}\quad
 (\forall k\in N^\In(v))\notag
\end{equation}
while all other weights are unchanged.  The DDP-Normalized networks
are invariant to the above transformation because the output of each
node is normalized.  The SGD update rule is however not invariant to
this transformation:
\begin{align*}
T(\tilde{w})^{+}_{k\rightarrow v} &= \rho \tilde{w}_{k\rightarrow 
v}-\eta\frac{\partial L}{\rho \partial \tilde{w}_{k\rightarrow 
v}}(\tilde{\vecW})
\neq \rho\left(\tilde{w}_{k\rightarrow v}-\eta \frac{\partial L}{\partial 
\tilde{w}_{k\rightarrow v}}(\tilde{\vecW})\right) = \rho 
\tilde{w}^{+}_{k\rightarrow v}\\
\end{align*}

\section{Understanding Invariances} \label{sec:rescaling}
The goal of this section is to discuss whether being invariant to node-wise
rescaling transformations is sufficient or not.

Ideally we would like our algorithm to be at least invariant to all the transformations to which the model $G$ is invariant. Note that this is different than the invariances studied in \cite{ollivier2015riemannian}, in that they study algorithms that are invariant to reparametrizations of the same model but we look at transformations within the the parameter space that preserve the function in the model. This will eliminate the need for non-trivial initialization. Thus our goal is to characterize the whole variety of transformations to which the model is invariant and check if the algorithm is invariant to all of them.

We first need to note that invariance can be composed. If a 
network $G$ is invariant to transformations $T_1$ and
$T_2$, it is also invariant to their composition $T_1\circ T_2$. This is
also true for an algorithm. If an algorithm is invariant to
transformations $T_1$ and $T_2$, it is also invariant to their
composition. This is because $f_{T_2\circ T_1\circ
\mathcal{A}(\vw)}=f_{T_2\circ \mathcal{A}(T_1\circ
\vw)}=f_{\mathcal{A}(T_2\circ T_1(\vw))}$.

Then it is natural to talk about the {\em basis} of invariances. The
intuition is that although there are infinitely many transformations to which the
model (or an algorithm) is invariant, they could be generated as  
compositions of finite number of transformations.

In fact, in the infinitesimal limit the directions of infinitesimal
changes in the parameters to which the function $f_{\vw}$ is
insensitive form a subspace. This is because for a  fixed input $\vx$,
we have
\begin{align}
\label{eq:taylor-expansion}   
 f_{\vw+\vecDelta}(x) = f_{\vw}(\vx) + \sum\nolimits_{e\in E}\frac{\partial f_{\vw}(\vx)}{\partial w_{e}}\cdot \Delta_{e} + O(\|\vecDelta\|^2),
\end{align}
where $E$ is the set of edges,
due to a Taylor expansion around $\vw$. Thus the function $f_{\vw}$
is insensitive (up to $O(\|\vecDelta\|^2)$) to any change in
the direction $\vecDelta$ that lies in the (right) null space of the
Jacobian matrix $\partial
f_{\vw}(\vx)/\partial \vw$ for all input $\vx$
simultaneously. More formally, the subspace can be defined as
\begin{align}
\label{eq:defN}
 N(\vw) = \bigcap\nolimits_{\vx\in\RR^{|V_{\rm in}|}}\textrm{Null}\left(\frac{\partial
 f_{\vw}(\vx)}{\partial \vw}\right).
\end{align}
Again, any change to $\vw$ in the direction $\vecDelta$ that lies in
$N(\vw)$ leaves the function $f_{\vw}$ unchanged (up to $O(\|\vecDelta\|^2)$) at {\em
every} input $x$.
Therefore, if we can calculate the dimension of $N(\vw)$ and if we
have ${\rm dim}N(\vw)= |V_{\rm internal}|$, where 
we denote the number of internal nodes by $|V_{\rm internal}|$, then we can conclude that all
infinitesimal transformations to which the model is invariant can be
spanned by infinitesimal node-wise rescaling transformations.

Note that the null space $N(\vw)$ and its dimension is a function of
$\vw$. Therefore, there are some points in the parameter space that
have more invariances than other points. For example,
suppose that $v$ is an internal node with ReLU activation that receives
connections only from other ReLU units (or any unit whose output is
nonnegative).
If all the incoming weights to 
$v$ are negative including the bias, the output of node $v$ will be zero
regardless of the input, and
the function $f_{\vw}$ will be insensitive to any transformation to
the outgoing weights of $v$.
Nevertheless we conjecture that
 as the network size grows, the chance of
being in such a degenerate configuration during training will diminish exponentially.

When we study the dimension of $N(\vw)$, it is convenient to analyze the
dimension of the span of the row vectors of
the Jacobian matrix $\partial
f_{\vw}(\vx)/\partial \vw$ instead. We define the degrees of freedom of
model $G$ at $\vw$ as
 \begin{align}
  \label{eq:dof}
 d_G(\vw) = {\rm dim}\left(\bigcup\nolimits_{\vx\in\RR^{|V_{\rm in}|}}{\rm Span}\left(
\frac{\partial f_{\vw}(\vx)}{\partial \vw}[v,:] : v\in V_{\rm out}
 \right)\right),
 \end{align}
where $\partial f_{\vw}(\vx)[v,:]/\partial \vw $ denotes the $v$th
row vector of the Jacobian matrix and $\vx$ runs over all possible input $\vx$.
Intuitively, $d_{G}(\vw)$ is the dimension of the set of directions that
changes $f_{\vw}(x)$ for at least one input $x$.

Due to the rank nullity theorem $d_G(\vw)$ and the dimension of $N(\vw)$
are related as follows:
\begin{align*}
 d_G(\vw) + {\rm dim}\left(N(\vw)\right)=|E| ,
\end{align*}
where $|E|$ is the number of parameters. Therefore, again if $d_G(\vw)=|E| -
|V_{\rm internal}|$, then we can conclude that infinitesimally speaking,
all transformations to which the
model is invariant can be spanned by node-wise rescaling
transformations.

Considering only invariances that hold uniformly over all input $\vx$
could give an under-estimate of the class of invariances, i.e., there might be some invariances that
hold for many input $\vx$ but not all. An alternative approach
for characterizing invariances is to define a measure of distance
between functions that the neural network model represents based on 
the input distribution, and 
infinitesimally study the subspace of directions to which the distance
is insensitive. We can define distance between two functions $f$
and $g$ as
\begin{align*}
 D(f,g) = \EE_{\vx\sim \mathcal{D}}\left[m(f(\vx),g(\vx))\right],
\end{align*}
where $m:\RR^{|V_{\rm out}|\times |V_{\rm out}|}\rightarrow \RR$ is a
(possibly asymmetric) distance measure between two vectors
$\vz,\vz'\in\RR^{|V_{\rm out}|}$, which we require that
$m(\vz,\vz)=0$ and $\partial m/\partial \vz'_{\vz=\vz'}=0$. For example, $m(\vz,\vz')=\|\vz-\vz'\|^2$.

The second-order Taylor expansion of the distance $D$ can be written as
\begin{align*}
  D(f_{\vw}\| f_{\vw+\vecDelta}) &=\frac{1}{2}
\vecDelta\T \cdot F(\vw)\cdot
  \vecDelta
 +o(\|\vecDelta\|^2),
\end{align*}
where
\begin{align*}
 F(\vw)&=\EE_{\vx\sim \mathcal{D}}\left[\left(\frac{\partial
 f_{\vw}(\vx)}{\partial \vw}\right)\T
 \cdot\left.\frac{\partial^2
m(\vz,\vz')}{\partial \vz'^2}\right|_{\vz=\vz'=f_{\vw}(\vx)} 
\cdot\left(\frac{\partial f_{\vw}(\vx)}{\partial \vw}\right)\right]
 \end{align*}
and $\partial^2 m(\vz,\vz')/\partial \vz'^2|_{\vz=\vz'=f_{\vw}(\vx)}$ is the Hessian of
the distance measure $m$ at $\vz=\vz'=f_{\vw}(\vx)$.

Using the above expression, we can define the input distribution
dependent version of $N(\vw)$ and $d_{G}(\vw)$ as
 \begin{align*}
  N_{\mathcal{D}}(\vw) = {\rm Null} F(\vw),\qquad
 d_{G,\mathcal{D}}(\vw) = {\rm rank} F(\vw).
 \end{align*}
Again due to the rank-nullity theorem we have
$d_{G,\mathcal{D}}(\vw)+{\rm dim}(N_{\mathcal{D}}(\vw))=|E|$.

As a special case, we obtain the Kullback-Leibler
divergence $D_{\rm KL}$, which is commonly considered as {\em the} way
to study invariances, by 
choosing $m$ as the
conditional Kullback-Leibler divergence of output $y$ given the network output
as
\begin{align*}
 m(\vz,\vz') = \EE_{y\sim q(y|\vz)}\left[\log\frac{q(y|\vz)}{q(y|\vz')}\right],
\end{align*}
where $q(y|\vz)$ is a link function, which can be, e.g., the soft-max
$q(y|\vz)=e^{z_y}/\sum_{y'=1}^{|V_{\rm out}|}e^{z_{y'}}$.
However, note that the
invariances in terms of $D_{\rm KL}$ depends not only on the input
distribution but also on the choice of the link function $q(y|\vz)$.

%
%

\subsection{Path-based characterization of the network}
\label{sec:path-network}
A major challenge in studying the degrees of freedom \eqref{eq:dof}
 is the fact that the
Jacobian $\partial f_{\vw}(x)/\partial \vw$ depends on both parameter $w$
and input $x$. In this section, we first tease apart the two
dependencies by rewriting $f_{\vw}(x)$ as 
the sum over all
directed paths from every input node to each output node as follows:
 \begin{align}
 \label{eq:f-as-sum-over-paths}
 f_{\vw}(\vx)[v] &= \sum\nolimits_{p\in\Pi(v)}g_{p}(\vx)\cdot\pi_p(\vw)\cdot x[{\rm 
head}(p)],
 \end{align}
where $\Pi(v)$ is the set of all directed path from any 
input
 node to $v$, ${\rm head}(p)$ is the first node of path $p$,
  $g_{p}(\vx)$ takes 1 if all the rectified
 linear units along path $p$ is active and zero otherwise, and
$\pi_p(\vw)=\prod_{e\in E(p)} w(e)$
is the product of the weights along path $p$; $E(p)$ denotes the set 
of edges that appear along path $p$.

Let $\Pi=\cup_{v\in V_{\rm out}}\Pi(v)$ be the set of all directed paths.
We define the path-Jacobian matrix $J(\vw)\in\RR^{|\Pi|\times |E|}$
as $J(\vw)=(\partial \pi_p(\vw)/\partial w_e)_{p\in \Pi, e\in E}$.
In addition, we define
$\vphi(\vx)$ as a $|\Pi|$ dimensional vector  with $g_p(\vx)\cdot x[{\rm
head}(p)]$ in the corresponding entry.
The Jacobian of the network $f_{\vw}(\vx)$ can now be expressed as
 \begin{align}
  \label{eq:jacobian}
\frac{\partial f_{\vw}(\vx)[v]}{\partial \vw}&= J_v(\vw)\T \vphi_v(\vx),
 \end{align}
where where $J_v(\vw)$ and $\vphi_v(\vx)$ are the submatrix (or subvector) of
 $J(\vw)$ and $\vphi(\vx) $ that corresponds to output
node $v$, respectively\footnote{Note that although path activation $g_p(\vx)$ is a function of $\vw$,
it is insensitive to an infinitesimal change in the parameter, unless
the input to one of the rectified linear activation functions
along path $p$ is at exactly zero, which happens with probability
zero. Thus we treat $g_p(\vx)$ as constant here.}. Expression \eqref{eq:jacobian} 
clearly separates the dependence to the parameters $\vw$ and input $\vx$.

Now we have the following statement (the proof is given in Appendix \ref{sec:proofs}).
 \begin{theorem}
  \label{thm:dof}
The degrees-of-freedom $d_{G}(\vw)$ of neural network model $G$ is
at most the rank of the path Jacobian matrix $J(\vw)$.
 The equality holds if ${\rm dim}\left({\rm
 Span}(\vphi(\vx):\vx\in\RR^{|V_{\rm in}|})\right)=|\Pi|$; i.e. when
 the dimension of the space
 spanned by  $\vphi(\vx)$
equals the total  number of paths $|\Pi|$.
 \end{theorem}

An analogous statement holds for the input distribution dependent
degrees of freedom $d_{G,\mathcal{D}}(\vw)$, namely,
$d_{G,\mathcal{D}}(\vw)\leq {\rm rank} J(\vw)$ and the equality holds if
the rank of the $|\Pi|\times |\Pi|$ path covariance matrix
$(\EE_{\vx\sim\mathcal{D}}\left[
\partial^2 m(\vz,\vz')/\partial z'_v\partial z'_{v'}\phi_{p}(\vx)\phi_{p'}(\vx)
\right])_{p,p'\in \Pi}$ is full, where $v$ and $v'$ are the end nodes of
paths $p$ and $p'$, respectively.

It remains to be understood when the dimension of the span of the path
vectors $\vphi(\vx)$ become full. The answer depends on
$\vw$. Unfortunately, there is no typical behavior as we know from the
example of an internal ReLU unit connected to ReLU units by negative
weights. In fact, we can choose any number of internal units in the network to be
in this degenerate state creating different degrees of degeneracy.
Another way to introduce degeneracy is to insert a linear layer in the
network. This will superficially increase the number of paths but will
not increase the dimension of the span of $\vphi(\vx)$. 
For example,
consider a linear classifier $z_{\rm out}=\inner{\vw}{\vx}$ with $|V_{\rm in}|$ inputs. If
the whole input space is spanned by $\vx$, the dimension of the span of
$\vphi(\vx)$ is $|V_{\rm in}|$, which agrees with the number of paths. Now let's insert a linear layer with units
$V_1$ in between the input and the output layers. The number of paths has increased from $|V_{\rm in}|$ to $|V_{\rm
in}|\cdot|V_1|$. However the dimension of the span of $\vphi(\vx)={\vec
1}_{|V_1|}\otimes \vx$ is still $|V_{\rm in}|$, because the linear units
are always active.
Nevertheless we conjecture that 
there is a configuration $\vw$ such that ${\rm dim}\left({\rm
 Span}(\vphi(\vx):\vx\in\RR^{|V_{\rm in}|})\right)=|\Pi|$ and the set of
 such $\vw$ grows as the network becomes larger.

%
%
%
%
%

\subsection{Combinatorial characterization of the rank of path Jacobian} 
 Finally, we show that the rank of the path-Jacobian matrix $J(\vw)$ is
determined purely combinatorially by the graph $G$ 
except a subset of the parameter space with zero Lebesgue measure. The
proof is given in Appendix \ref{sec:proofs}.

 \begin{theorem}
  \label{thm:rank-path-jacobian}
The rank of the path Jacobian matrix $J(\vw)$ is generically (excluding set
  of parameters with zero Lebesgue measure) equal to the number of
 parameters $|E|$ minus the number of internal nodes of the network.
 \end{theorem}

Note that the dimension of the space spanned by node-wise rescaling
\eqref{eq:node-wise} equals the number of internal nodes. Therefore,
node-wise rescaling is the {\em only} type of invariance for a ReLU
network with fixed architecture $G$, if ${\rm dim}\left({\rm
Span}(\phi(\vx):\vx\in\RR^{|V_{\rm in}|})\right)=|\Pi|$ at parameter $\vw$.
 
As an example, let us consider a simple 3 layer network with 2 nodes in each layer
except for the output layer, which has only 1 node (see 
Figure~\ref{fig:net2221}). The network has 10
parameters (4, 4, and 2 in each layer respectively) and 8
paths. The Jacobian $(\partial f_{\vw}(\vx)/\partial \vw)$ can be written as
$(\partial f_{\vw}(\vx)/\partial \vw)
= J(\vw)\T\cdot \vphi(\vx)$, where
\begin{align}
 \label{eq:J-2221}
 J(\vw) &=
 \left[
 \begin{array}{c|c|c}
\begin{array}{cccc}
 w_5 w_9 & & & \\
 & w_5w_9 & & \\
 & & w_6w_9 & \\
 & & & w_6w_9 \\
 \hline
 w_7w_{10} & & & \\
 & w_7w_{10} & & \\
 & & w_8w_{10} & \\
 & & & w_8w_{10} \\
\end{array}
&
\begin{array}{cccc}
w_9w_1 & & & \\
w_9w_2 & & & \\
 & w_9w_3 & & \\
 & w_9w_4 & &\\
 \hline
  & & w_{10}w_1 & \\
 & & w_{10}w_2 & \\
 & & & w_{10}w_3 \\
 & & & w_{10}w_4
\end{array}
& \begin{array}{cc}
 w_5w_1 & \\
 w_5w_2 & \\
 w_6w_3 & \\
 w_6w_4 & \\
\hline
 & w_7w_1 \\
 & w_7w_2 \\
 & w_8w_3 \\
 & w_8w_4
\end{array}
 \end{array}
\right]
 \intertext{and}
 \phi(\vx)^\top&=
\begin{bmatrix}
 g_1(\vx)x[1] & g_2(\vx)x[2]  & g_3(\vx)x[1] & g_4(\vx)x[2] &  g_5(\vx)x[1] & 
g_6(\vx)x[2]&  g_7(\vx)x[1] &  g_8(\vx)x[2] 
\end{bmatrix}.\notag
\end{align}
The rank of $J(\vw)$ in \eqref{eq:J-2221} is (generically) equal to $10-4=6$, which is smaller
than both the number of parameters and the number of paths.
 \begin{figure}[htb]
\begin{center}
 \includegraphics[clip,width=.28\textwidth]{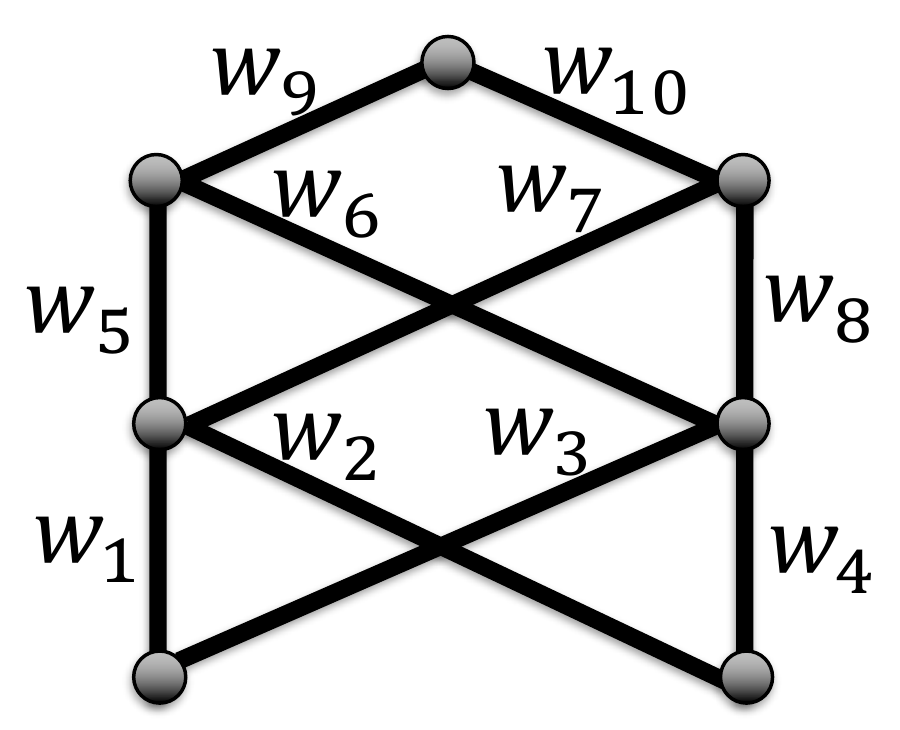}
\end{center}
\caption{A 3 layer network with 10 parameters and 8 paths.}
\label{fig:net2221}
 \end{figure}

\section{Conclusion and Future Work}
We proposed a unified framework as a complexity measure or 
regularizer for neural networks and discussed normalization and 
optimization with respect to this regularizer. We further showed how 
this measure interpolates between data-dependent and data-independent 
regularizers and discussed how Path-SGD and Batch-Normalization are 
special cases of optimization with respect to this measure. We also 
looked at the issue of invariances and brought new insights to this 
area.
\bibliographystyle{iclr2016_conference}
\bibliography{iclr2016}
\appendix

\section{Implementation}
\subsection{DDP-Normalization}\label{sec:BN-imp}
Given any batch of $n$ data points to estimate mean, variance and the 
gradient, the stochastic gradients for the weight $\tilde{\vecW}$ (weights 
in the DDP-Normalized network) can then be calculated through the 
chain rule:
\begin{align}\label{eq:dl}
\frac{\partial L}{\partial \tilde{\vecW}_{\rightarrow v}} 
&=\frac{1}{n\tilde{\gamma}_v}\sum_{i=1}^{n} \frac{\partial L}{\partial 
z_{v}^{(i)}}\left[ {\vecH}_{N^{\In}(v)}^{(i)} - \frac{1}{n} 
\sum_{j=1}^n {\vecH}_{N^{\In}(v)}^{(j)} -\frac{
\hat{z}^{(i)}_v }{ 
2\tilde{\gamma}_v^2}\frac{\partial \tilde{\gamma}^2_v}{\partial \tilde{\vecW}_{\rightarrow 
v}} \right]\\
\frac{\partial L}{\partial z^{(i)}_{u}} &= \frac{1}{\tilde{\gamma}_v} 
\left[\sum_{v \in N^{\Out}(u)} \tilde{w}_{u\rightarrow 
v}\left(\frac{\partial L}{\partial z_{v}^{(i)}} - 
\frac{1}{n}\sum_{j=1}^n\frac{\partial L}{\partial 
z_{v}^{(j)}}\left(1-\alpha\frac{\hat{z}^{(i)}_v\hat{z}^{(j)}_v}{\tilde{\gamma}^2_v}\right)\right)\right]_{z^{(i)}_{u}\geq 0}\\ \notag
\end{align}
where $\hat{z}^{(i)}_{v} = \tilde{z}^{(i)}_{v} - \frac{1}{n} 
\sum_{j=1}^n \tilde{z}^{(j)}_{v}$ and we have:
\begin{equation}\label{eq:dgam}
\frac{\partial \tilde{\gamma}^2_v}{\partial \tilde{\vecW}_{\rightarrow 
v}}=2(1-\alpha)\tilde{\vecW}_{\rightarrow v} + 
\frac{2\alpha}{n}\sum_{i=1}^n \hat{z}^{(i)}_v \left({\vecH}_{N^{\In}(v)}^{(i)} - \frac{1}{n} 
\sum_{j=1}^n {\vecH}_{N^{\In}(v)}^{(j)}\right)
\end{equation}
Similar to Batch-Normalization, all the above calculations can be efficiently 
carried out as vector operations with negligible extra memory and computations.
\subsection{DDP-SGD}\label{sec:ddp-imp}
In order to compute the second derivatives 
$\kappa_e(\vecW)=\frac{\partial^2 \gamma^2_{\rm net}}{\partial w^2_e}$, 
we first calculate the first derivative. The backpropagation can be 
done through $\gamma^2_u$ and  $z^{(i)}_{u}$ but this makes it 
difficult to find the second derivatives. Instead we propagate the 
loss through $\gamma^2_u$ and the second order terms of the form 
$z^{(i)}_{u_1}z^{(i)}_{u_2}$:
\begin{align}
\frac{\partial \gamma^2_{{\rm net}}}{\partial \gamma^2_u} &= 
(1-\alpha)\sum_{v\in N^\Out(u)} \frac{\partial \gamma^2_{{\rm 
net}}}{\partial \gamma^2_v} w^2_{u\rightarrow v}
\end{align}
\begin{equation}
\frac{\partial \gamma^2_{{\rm net}}}{\partial 
(z_{u_1}^{(i)}z_{u_2}^{(i)})} = \alpha \left[\frac{ \partial 
\gamma^2_{\rm net}}{\partial \gamma^2_{u_1}}\right]_{u_1=u_2} + 
\left[\sum_{(v_1,v_2)\in \left(N^\Out(u_1)\right)^2} \frac{\partial 
\gamma^2_{\rm net}}{\partial 
(z_{v_1}^{(i)}z_{v_2}^{(i)})}w_{u_1\rightarrow v_1} w_{u_2 
\rightarrow v_2} \right]_{z^{(i)}_{u_1}>0, z^{(i)}_{u_2}>0}
\end{equation}
Now we can calculate the partials for $w_{u\rightarrow v}$ as follows:
\begin{equation}
\frac{\partial \gamma^2_{\rm net}}{\partial w_{u\rightarrow v}} = 
2(1-\alpha)\frac{\partial \gamma^2_{\rm net}}{\partial \gamma^2_v} 
\gamma^2_u w_{u\rightarrow v} + 2\sum_{i=1}^n\sum_{v'\in N^\Out(u)} 
\frac{\partial \gamma^2_{\rm net}}{\partial (z^{(i)}_v z^{(i)}_{v'})} 
h_{u}^{(i)}z_{v'}^{(i)}
\end{equation}
Since the partials $\frac{\partial \gamma^2_{{\rm net}}}{\partial 
\gamma^2_u}$ and $\frac{\partial \gamma^2_{{\rm net}}}{\partial 
(z_{u_1}^{(i)}z_{u_2}^{(i)})}$ do not depend on $w_{u\rightarrow v}$, 
the second order derivative can be calculated directly:
\begin{equation}
\kappa_{u\rightarrow v}(\vecW)=\frac{1}{2}\frac{\partial^2 \gamma^2_{\rm 
net}}{\partial w_{u\rightarrow v}^2} = (1-\alpha)\frac{\partial 
\gamma^2_{\rm net}}{\partial \gamma^2_v} \gamma^2_u + 
\sum_{i=1}^n\frac{\partial \gamma^2_{\rm net}}{\partial 
\left({z^{(i)}_v}^2\right)}\left(h^{(i)}_u\right)^2
\end{equation}

\section{Natural Gradient}\label{sec:ng}
The natural gradient algorithm \citep{Ama98} achieves invariance 
by applying the inverse of the Fisher information
matrix $F(\vw^{(t)})$ at the current parameter $\vw^{(t)}$ to the 
negative gradient direction
as follows:
 \begin{align}
 \vw^{(t+1)} &=  \vw^{(t)} + \eta\vecDelta^{(\rm natural)},\notag
 \intertext{where}
 \label{eq:ng-argmin}
 \Delta^{(\rm natural)} &=
   \argmin{\Delta\in\RR^{|E|}}
  \inner{-\frac{\partial L}{\partial w}(\vw^{(t)})}{\vecDelta}
  ,\quad {\rm s.t.}\quad \vecDelta\T F(\vw^{(t)})\vecDelta\leq \delta^2\\
 \label{eq:ng}
&=-F^{-1}(\vw^{(t)})\frac{\partial L}{\partial w}(\vw^{(t)}).
 \end{align}
Here $F(\vw)$ is the Fisher information matrix at point $\vw$ and is
defined with respect to the probabilistic view of the feedforward neural network
model, which we describe in more detail below.

Suppose that we are solving a classification problem and the final 
layer
of the network is fed into a softmax layer that determines the
probability of candidate classes given the input $x$. Then the neural
network with the softmax layer can be viewed as a conditional
probability distribution
 \begin{align}
  \label{eq:cond-prob}
 q(y|\vx)= \frac{\exp(f_{\vecW}(\vx)[v_y])}{\sum_{v\in  V_{\rm 
out}}\exp(f_{\vecW}(\vx)[v])},
 \end{align}
where $v_y$ is the output node corresponding to class $y$.
If we are solving a regression problem a Gaussian distribution is
probably more appropriate for $q(y|\vx)$.

Given the conditional probability distribution $q(y|\vx)$, the Fisher
information matrix can be defined as follows:
 \begin{align}
  \label{eq:fisher-information}
 F(\vw)[e,e'] = \EE_{\vx\sim p(\vecX)}\EE_{y\sim q(y|\vx)}\left[
\frac{\partial \log q(y|\vx)}{\partial w_e}
\frac{\partial \log q(y|\vx)}{\partial w_{e'}}
 \right],
 \end{align}
where $p(x)$ is the marginal distribution of the data.

Since we have
 \begin{align}
  \label{eq:partial-logq}
 \frac{\partial \log q(y|\vx)}{\partial w_{u\rightarrow v}}
 =\frac{\partial \log q(y|\vx)}{\partial z_v}\cdot h_u
 =\sum_{v'\in V_{\rm out}}
 \frac{\partial \log q(y|\vx)}{\partial z_{v'}}\cdot
 \frac{\partial z_{v'}}{\partial z_{v}}\cdot h_u
 \end{align}
using the chain rule, each entry of the Fisher information matrix can 
be
computed efficiently by forward and backward propagations on a 
minibatch.

\section{Proofs}
\label{sec:proofs}
\begin{proof}[Proof of Theorem \ref{thm:orthogonal}]
First note that we can calculate the following inner product using equation~\eqref{eq:dgam}:
\begin{align*}
\inner{\tilde{\vecW}_{\rightarrow v}}{\frac{ \partial \tilde{\gamma}_v^2}{ \partial \tilde{\vecW}_{\rightarrow v}}}
&= 2(1-\alpha)\norm{ \tilde{\vecW}_{\rightarrow v} }_2^2 + 
\frac{2\alpha}{n}\sum_{i=1}^n (\hat{z}^{(i)})^2 \\
&= 2(1-\alpha)\norm{ \tilde{\vecW}_{\rightarrow v} }_2^2 + 
2\alpha \var(\tilde{z}_v) = 2\tilde{\gamma}_v^2
\end{align*}
Next, by equation \eqref{eq:dl} we get:
\begin{align*}
\inner{ \tilde{\vecW}_{\rightarrow v}}{ \frac{\partial L}{\partial \tilde{\vecW}_{\rightarrow v}} }
&=\frac{1}{n\tilde{\gamma}_v}\sum_{i=1}^{n} \frac{\partial L}{\partial 
z_{v}^{(i)}}\left[ \hat{z}^{(i)}_v -\frac{
\hat{z}^{(i)}_v }{ 
2\tilde{\gamma}_v^2}\inner{\tilde{\vecW}_{\rightarrow v} }{ \frac{\partial \tilde{\gamma}^2_v}{\partial \tilde{\vecW}_{\rightarrow 
v}}} \right]\\
&= \frac{1}{n\tilde{\gamma}_v}\sum_{i=1}^{n} \frac{\partial L}{\partial 
z_{v}^{(i)}}\left[ \hat{z}^{(i)}_v -\frac{
\hat{z}^{(i)}_v }{ 
2\tilde{\gamma}_v^2} 2 \tilde{\gamma}^2_v \right] = 0
\end{align*}
\end{proof}
\begin{proof}[Proof of Theorem \ref{thm:dof}]
First we see that \eqref{eq:jacobian} is true because
 \begin{align*}
\frac{\partial f_{\vw}(\vx)[v]}{\partial \vw} =\Bigl(\sum_{p\in\Pi(v)}
\frac{\partial \pi_p(\vw)}{\partial w_e}
 \cdot g_{p}(\vx)\cdot x[{\rm head}(p)]\Bigr)_{e\in E}
 = J_{v}(\vw)\T\cdot \phi_{v}(\vx).
\end{align*}
 Therefore, 
 \begin{align}
\bigcup_{\vx\in\RR^{|V_{\rm in}|}} {\rm Span}\left(
\frac{\partial f_{\vw}(\vx)[v]}{\partial \vw}: v\in V_{\rm out}
 \right)
&= \bigcup_{\vx\in\RR^{|V_{\rm in}|}}{\rm Span}\left(
J_{v}(\vw)\T\cdot \phi_{v}(\vx): v\in V_{\rm out}
  \right)\notag\\
  \label{eq:span}
 &=J(\vw)\T \cdot {\rm Span}\left(\phi(\vx): \vx\in\RR^{|V_{\rm in}|}\right).
\end{align}
Consequently,  any vector of the form $(\frac{\partial
f_{\vw}(\vx)[v]}{\partial w_e})_{e\in E}$ for a fixed input $\vx$ lies in the
 span of the row vectors of the path Jacobian $J(\vx)$.

The second part says $d_{G}(\vw)={\rm rank}J(\vw)$ if ${\rm
dim}\left({\rm Span}(\phi(\vx):\vx\in\RR^{|V_{\rm in}|})\right)=|\Pi|$, which is the number
of rows of $J(\vw)$. We can see that this is true from expression \eqref{eq:span}.

\end{proof}

\begin{proof}[Proof of Theorem \ref{thm:rank-path-jacobian}]
First, $J(\vw)$  can be written as an Hadamard product between path 
incidence
 matrix $M$ and a rank-one matrix as follows:
\begin{align*}
  J(\vw) &= M \circ \left(\vw^{-1} \cdot \vpi^\top(\vw)\right),
\end{align*}
 where $M$ is the path incidence matrix whose $i,j$ entry is one if
the $i$th edge is part of the $j$th path,  $\vw^{-1}$ is an entry-wise
 inverse of the parameter vector $\vw$, 
$\vpi(\vw)=(\pi_p(\vw))$ is a vector containing the product along each 
path in
 each entry, and $\top$ denotes transpose.

 Since we can rewrite
\begin{align*}
  J(\vw) &= {\rm diag}(\vw^{-1})\cdot M \cdot{\rm diag}(\vpi(\vw)),
\end{align*}
 we see that (generically) the rank of $J(\vw)$ is equal to the rank of
 zero-one matrix $M$.

 Note that the rank of $M$ is equal to the number of linearly
 independent columns of $M$, in other words, the number of linearly
 independent paths. In general, most paths are not independent. For
 example, in Figure \ref{fig:net2221}, we can see that the column
 corresponding to the path
 $w_2w_7w_{10}$ can be produced by combining 3 columns corresponding to
 paths $w_1w_5w_9$,  $w_1w_7w_{10}$, and $w_2w_5w_9$.

 In order to count the number of independent paths, we use 
mathematical
 induction. For simplicity, consider a layered graph with $d$
 layers. All the edges from
 the $(d-1)$th layer nodes to the output layer nodes are linearly 
independent,
 because they correspond to different parameters. So far we have
 $n_dn_{d-1}$ independent paths.

 Next, take one node $u_0$
 (e.g., the leftmost node) from the $(d-2)$th layer. All the paths 
starting
 from this node through the layers above are linearly
 independent. However, other nodes in this layer only contributes
 linearly to the number of independent paths. This is the case because
we can take an edge $(u,v)$, where $u$ is one of the remaining 
$n_{d-2}-1$
 vertices in the $(d-2)$th layer and $v$ is one of the $n_{d-1}$ 
nodes in
 the $(d-1)$th layer, and we can take any path (say $p_0$) from there 
to
 the top  layer. Then this is the only independent path that uses the
 edge $(u,v)$, because any other combination of edge $(u,v)$ and path
 $p$ from $v$ to the top layer can be
 produced as follows (see Figure \ref{fig:dependence}):
 \begin{align*}
  (u,v)\rightarrow p = (u,v)\rightarrow p_0 - (u_0,v)\rightarrow p_0 +
  (u_0,v)\rightarrow p.
 \end{align*}
 Therefore after considering all nodes in the $d-2$th layer, we have
 \begin{align*}
  n_{d}n_{d-1} + n_{d-1}(n_{d-2}-1)
 \end{align*}
 independent paths. Doing this calculation inductively, we have
 \begin{align*}
  n_{d}n_{d-1} + n_{d-1}(n_{d-2}-1) + \cdots + n_{1}(n_0-1)
 \end{align*}
 independent paths, where $n_0$ is the number of input units. This
 number is clearly equal to the number of parameters
 ($n_dn_{d-1}+\cdots+ n_{1}n_0$) minus the number of internal nodes 
($n_{d-1}+\cdots+n_1$).
\end{proof}
\begin{figure}[thb]
 \begin{center}
  \includegraphics[clip,width=0.7\textwidth]{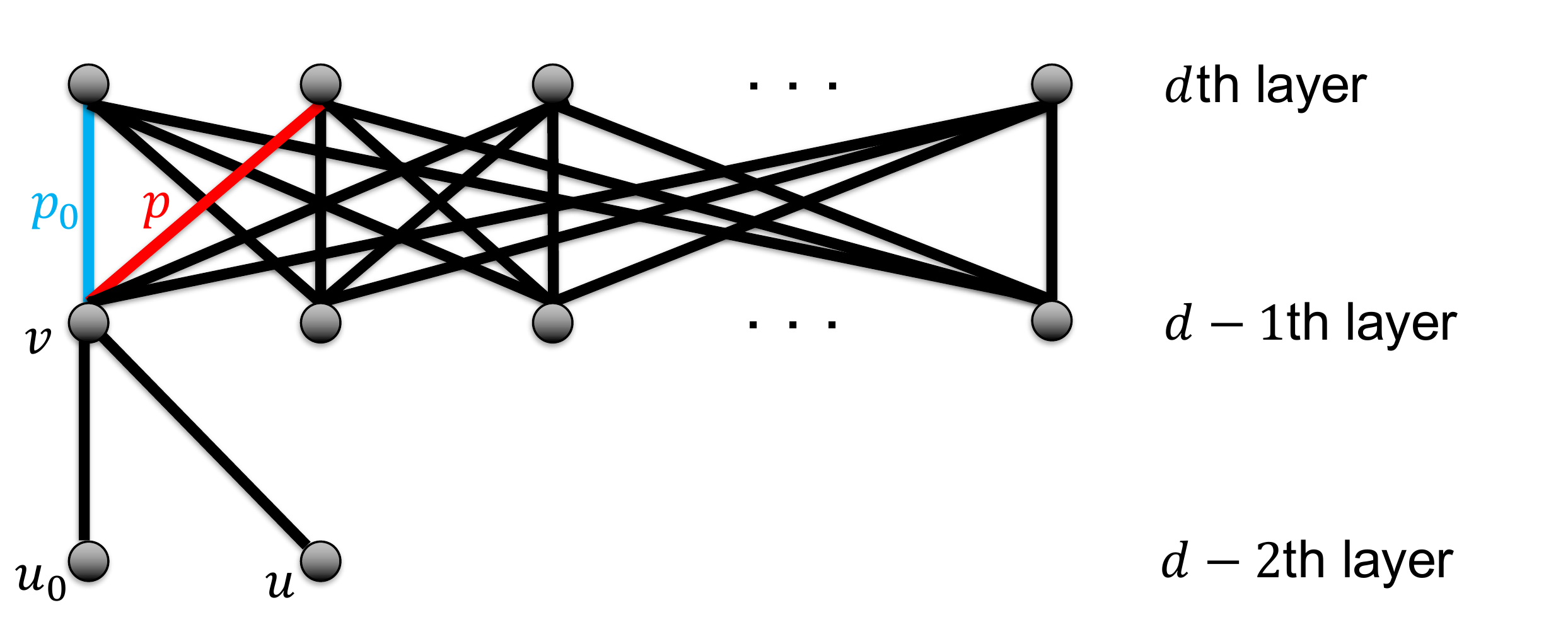}
 \end{center}
  \caption{Schematic illustration of the linear dependence of the four
  paths $(u_0,v)\rightarrow p_0$, $(u_0,v)\rightarrow p$,
  $(u,v)\rightarrow p_0$, and $(u,v)\rightarrow p$. Because of this
  dependence, any additional edge $(u,v)$ only contributes one
  additional independent path.}
  \label{fig:dependence}
\end{figure}

\end{document}

%% file: header.tex
\usepackage{times,wrapfig,amsmath,amsfonts,bm,color,enumitem,algorithm,algpseudocode}

\newcommand{\argmin}[1]{\underset{#1}{\mathrm{argmin}} \:}

\newcommand{\inner}[2]{\left\langle #1, #2 \right\rangle}

\newcommand{\D}{\mathcal{D}}

\newcommand{\E}{\ensuremath{\mathbb{E}}}

\newcommand{\norm}[1]{\left\lVert{#1}\right\rVert}
\newcommand{\abs}[1]{\left\lvert{#1}\right\rvert}

\newcommand{\diag}{\operatorname{diag}}

\newcommand{\R}{\mathbb{R}}

\mathchardef\hyphen="2D

%% file: main.bbl
\begin{thebibliography}{19}
\providecommand{\natexlab}[1]{#1}
\providecommand{\url}[1]{\texttt{#1}}
\expandafter\ifx\csname urlstyle\endcsname\relax
  \providecommand{\doi}[1]{doi: #1}\else
  \providecommand{\doi}{doi: \begingroup \urlstyle{rm}\Url}\fi

\bibitem[Amari(1998)]{Ama98}
Amari, Shun-Ichi.
\newblock Natural gradient works efficiently in learning.
\newblock \emph{Neural computation}, 10\penalty0 (2):\penalty0 251--276, 1998.

\bibitem[Desjardins et~al.(2015)Desjardins, Simonyan, Pascanu, and
  Kavukcuoglu]{desjardins2015natural}
Desjardins, Guillaume, Simonyan, Karen, Pascanu, Razvan, and Kavukcuoglu,
  Koray.
\newblock Natural neural networks.
\newblock \emph{arXiv preprint arXiv:1507.00210}, 2015.

\bibitem[Glorot \& Bengio(2010)Glorot and Bengio]{difficulty}
Glorot, Xavier and Bengio, Yoshua.
\newblock Understanding the difficulty of training deep feedforward neural
  networks.
\newblock In \emph{AISTATS}, 2010.

\bibitem[Grosse \& Salakhudinov(2015)Grosse and
  Salakhudinov]{grosse2015scaling}
Grosse, Roger and Salakhudinov, Ruslan.
\newblock Scaling up natural gradient by sparsely factorizing the inverse
  {Fisher} matrix.
\newblock In \emph{ICML}, 2015.

\bibitem[Ioffe \& Szegedy(2015)Ioffe and Szegedy]{IofSze15}
Ioffe, Sergey and Szegedy, Christian.
\newblock Batch normalization: Accelerating deep network training by reducing
  internal covariate shift.
\newblock In \emph{ICML}, 2015.

\bibitem[Larochelle et~al.(2009)Larochelle, Bengio, Louradour, and
  Lamblin]{larochelle2009exploring}
Larochelle, Hugo, Bengio, Yoshua, Louradour, J{\'e}r{\^o}me, and Lamblin,
  Pascal.
\newblock Exploring strategies for training deep neural networks.
\newblock \emph{The Journal of Machine Learning Research}, 10:\penalty0 1--40,
  2009.

\bibitem[{Le Cun} et~al.(1998){Le Cun}, Bottou, Orr, and
  M{\"{u}}ller]{lecun-98x}
{Le Cun}, Yann, Bottou, L\'{e}on, Orr, Genevieve~B., and M{\"{u}}ller,
  Klaus-Robert.
\newblock Efficient backprop.
\newblock In \emph{Neural Networks, Tricks of the Trade}, Lecture Notes in
  Computer Science LNCS~1524. Springer Verlag, 1998.
\newblock URL \url{http://leon.bottou.org/papers/lecun-98x}.

\bibitem[LeCun et~al.(1998)LeCun, Bottou, Orr, and Muller]{lecun1998neural}
LeCun, Yann, Bottou, Leon, Orr, Genevieve~B, and Muller, Klaus-Robert.
\newblock Neural networks-tricks of the trade.
\newblock \emph{Springer Lecture Notes in Computer Sciences}, 1524\penalty0
  (5-50):\penalty0 7, 1998.

\bibitem[Martens(2010)]{martens2010deep}
Martens, James.
\newblock Deep learning via hessian-free optimization.
\newblock In \emph{ICML}, 2010.

\bibitem[Martens \& Grosse(2015)Martens and Grosse]{martens2015optimizing}
Martens, James and Grosse, Roger.
\newblock Optimizing neural networks with {Kronecker-factored} approximate
  curvature.
\newblock In \emph{ICML}, 2015.

\bibitem[Neyshabur et~al.(2015{\natexlab{a}})Neyshabur, Salakhutdinov, and
  Srebro]{NeySalSre15}
Neyshabur, Behnam, Salakhutdinov, Ruslan, and Srebro, Nathan.
\newblock {Path-SGD}: Path-normalized optimization in deep neural networks.
\newblock In \emph{NIPS}, 2015{\natexlab{a}}.

\bibitem[Neyshabur et~al.(2015{\natexlab{b}})Neyshabur, Tomioka, and
  Srebro]{NeyTomSre15}
Neyshabur, Behnam, Tomioka, Ryota, and Srebro, Nathan.
\newblock Norm-based capacity control in neural networks.
\newblock In \emph{COLT}, 2015{\natexlab{b}}.

\bibitem[Neyshabur et~al.(2015{\natexlab{c}})Neyshabur, Tomioka, and
  Srebro]{neyshabur15b}
Neyshabur, Behnam, Tomioka, Ryota, and Srebro, Nathan.
\newblock In search of the real inductive bias: On the role of implicit
  regularization in deep learning.
\newblock \emph{International Conference on Learning Representations (ICLR)
  workshop track}, 2015{\natexlab{c}}.

\bibitem[Ollivier(2015)]{ollivier2015riemannian}
Ollivier, Yann.
\newblock Riemannian metrics for neural networks i: feedforward networks.
\newblock \emph{Information and Inference}, 4\penalty0 (2):\penalty0 108--153,
  2015.

\bibitem[Pascanu \& Bengio(2014)Pascanu and Bengio]{pascanu2013revisiting}
Pascanu, Razvan and Bengio, Yoshua.
\newblock Revisiting natural gradient for deep networks.
\newblock In \emph{ICLR}, 2014.

\bibitem[Roux et~al.(2008)Roux, Manzagol, and Bengio]{roux2008topmoumoute}
Roux, Nicolas~L, Manzagol, Pierre-Antoine, and Bengio, Yoshua.
\newblock Topmoumoute online natural gradient algorithm.
\newblock In \emph{NIPS}, 2008.

\bibitem[Schaul et~al.(2013)Schaul, Zhang, and Lecun]{schaul2013no}
Schaul, Tom, Zhang, Sixin, and Lecun, Yann.
\newblock No more pesky learning rates.
\newblock In \emph{ICML}, 2013.

\bibitem[Sutskever et~al.(2013)Sutskever, Martens, Dahl, and
  Hinton]{sutskever2013importance}
Sutskever, Ilya, Martens, James, Dahl, George, and Hinton, Geoffrey.
\newblock On the importance of initialization and momentum in deep learning.
\newblock In \emph{ICML}, 2013.

\bibitem[Vinyals \& Povey(2011)Vinyals and Povey]{vinyals2011krylov}
Vinyals, Oriol and Povey, Daniel.
\newblock Krylov subspace descent for deep learning.
\newblock In \emph{ICML}, 2011.

\end{thebibliography}
